\documentclass[12pt]{article}
\usepackage{tikz}
\usepackage{amsmath}
\usepackage{amssymb, amsthm, adjustbox, mathtools}
\usepackage{float, graphicx}
\usepackage[margin=0.5in]{geometry}
\usepackage{sistyle}
\usepackage{hyperref}
\usepackage{cleveref}
\usepackage[backend=bibtex, 
    natbib=true, 
    bibstyle=verbose, citestyle=verbose,    
    doi=true, url=true,                     
    citecounter=true, citetracker=true,
    block=space, 
    backref=true, backrefstyle=two,
    abbreviate=false,
    isbn=true, maxbibnames=4, maxcitenames=4,
    style=alphabetic]{biblatex}
\bibliography{Involution_Learning.bib}
    
\SIthousandsep{,}

\newcommand{\tbf}{\textbf}

\newcommand{\R}{\mathbb{R}}

\newcommand{\Z}{\mathbb{Z}}

\newcommand{\I}{\mathbb{I}}

\newcommand{\E}{\mathbb{E}}

\newcommand{\ep}{\varepsilon}
\newcommand{\la}{\lambda}
\newcommand{\La}{\Lambda}
\newcommand{\ph}{\varphi}
 
\newcommand{\Ga}{\Gamma}
\newcommand{\al}{\alpha}

\newcommand{\tht}{\theta}

\newcommand{\Om}{\Omega}
\newcommand{\sig}{\sigma} 
\newcommand{\Sig}{\Sigma}

\newcommand{\beq}{\begin{equation}}
\newcommand{\eeq}{\end{equation}}

\newcommand{\cN}{\mathcal{N}}

\newcommand{\cD}{\mathcal{D}}

\newcommand{\rb}{\rbrace} 
\newcommand{\lb}{\lbrace}

\newcommand{\lan}{\langle} 
\newcommand{\ran}{\rangle}
\newcommand{\ol}{\overline} 
\newcommand{\ul}{\underline}

\newcommand{\ot}{\otimes}

\newcommand{\sbs}{\subseteq}

\newcommand{\tit}{\textit}

\newtheorem{thm}{Theorem}
\newtheorem{prop}{Proposition}

\newtheorem{lemma}{Lemma}
\newtheorem{defn}{Definition}
\newtheorem{coro}{Corollary}
\newtheorem{assm}{Assumption}
\theoremstyle{remark}
\newtheorem*{remark}{Remark}

\title{Learning Linear Symmetries in Data Using Moment Matching}

\author{Colin Hagemeyer}

\begin{document}

\maketitle

\abstract{It is common in machine learning and statistics to use symmetries derived from expert knowledge to simplify problems or improve performance, using methods like data augmentation or penalties. In this paper we consider the unsupervised and semi-supervised problems of learning such symmetries in a distribution directly from data in a model-free fashion. We show that in the worst case this problem is as difficult as the graph automorphism problem. However, if we restrict to the case where the covariance matrix has unique eigenvalues, then the eigenvectors will also be eigenvectors of the symmetry transformation. If we further restrict to finding orthogonal symmetries, then the eigenvalues will be either be $1$ or $-1$, and the problem reduces to determining which eigenvectors are which. We develop and compare theoretically and empirically the effectiveness of different methods of selecting which eigenvectors should have eigenvalue $-1$ in the symmetry transformation, and discuss how to extend this approach to non-orthogonal cases where we have labels}

\section{Introduction}
Symmetries exist throughout mathematics and the sciences. They sometimes appear explicitly in objects or data points, but even more commonly in the laws and distributions that produce those objects or data points. For example, in classical mechanics, the laws of motion are invariant under rotations and translations of space. It doesn't matter which direction you call ``up" and ``down:" the laws of physics remain the same. These symmetries can be exploited in various ways to make solving problems easier, such as producing conservation laws in physics.

Much research has been done in using these symmetries in machine learning. Previous related research has tended to focus on one or more specific goals:

\begin{enumerate}
\item Exploiting known symmetries in the data. Examples include Group-Equivariant Convolutional Networks \citep{cw:equivariant}, data augmentation methods such as image rotations, and even standard Convolution Neural Networks which implicitly use translational symmetry.
\item Implicitly learning (local) continuous symmetries by learning representations, such as variational autoencoders, and disentangled representation learning.
\item Using known symmetries in the underlying space which may not strictly hold for the data. Using harmonic analysis one can utilize the prior that the desired distributions will tend to be more symmetric. This approach is taken in \citep{diaconis:stats}\citep{kondor:group}.
\item Finding symmetries in the data \tit{points} themselves, such as on abstract grids like \citep{skh:perceptron}.
\item Learning identity preserving transformations using identity labels such as with \citep{slywt:faceid-gan} or giving the model pairs equivalent data points such as in \citep{Nima:lie_algebra}.
\item Learning symmetries using the associated differential equations as in the recent preprint \citep{lt:hidden}.
\end{enumerate}

In our approach, we will attempt to learn symmetries in data, similar to the goal 2, but we will look to solve the potentially more difficult problem of learning discrete symmetries. By this we mean that there is no way of continuously deforming the identity map into this transformation, in such a way that each subsequent deformation is still a symmetry. For example, the horizontal reflection of a face should still look like a face. This is a discrete symmetry, which we can see because handedness is binary. On the other hand, rotation of a face is still a face, and we can continuously deform a 90 degree rotation to a 0 degree rotation by just setting the degrees to be $d(t) = 90(1-t)$.

One of the key motivations of this question is that many symmetries used in data augmentation are discrete. This is especially the case when moving outside the the realm of images, such as in natural language processing where there is no way to continuously deform a sentence. In particular, this leads to the potential for a novel approach to semi-supervised learning, where a symmetries could be extracted from a large unlabeled dataset, and then these symmetries used to improve the performance of supervised learning on a smaller labeled dataset using data augmentation or other methods. 

\section{Basic Definitions}
First we start by defining formally what we mean by a symmetry in a distribution. To start with, by a probability space we will always mean a complete, Hausdorff space together with a probability measure on the Borel $\sigma$-algebra. We'll typically abuse notation and use $X$ to refer to both the probability space and the underlying topological space. If the reader is unfamiliar with this terminology, just consider the case of a probability distribution on $\R^n$, and $\mu(Z) = \int_Z p(x)dx$ is the probability of the event $Z$.

\begin{defn}
Let $X$ be a probability space, and let $f: X \to X$ be a continuous map. We say $X$ is \tbf{$f$-invariant} if $\mu(f^{-1}Z) = \mu(Z)$ for all measurable $Z \sbs X$. We will also say $f$ is a \ul{symmetry} of $X$.
\end{defn}

Typically we'll assume the distribution has a continuous PDF $p_\mu$ and $f$ is one-to-one and volume preserving, which makes this equivalent to $p_\mu(x) = p_\mu(f(x))$. We can also easily extended this definition to a full groups of symmetries:

\begin{defn}
Let $X$ be a probability space, and $G$ be a group acting on $X$ by continuous maps. We say $X$ is \tbf{$G$-invariant} if $\mu(g^{-1} \cdot Z) = \mu(Z)$ for all $g \in G$ and measurable $Z \sbs X$. We will say $G$ is a \tbf{symmetry group} of $X$.
\end{defn}

In this paper, we will focus on the case where $X$ is also a finite dimensional vector space. In this case, we want to further impose that any symmetry or symmetry-group action is linear, and in general we will use a more restrictive definition of linear as defined below.

\begin{defn}
Let $X$ be probability space which is also a finite dimensional vector space.
A \tbf{linear symmetry} is a symmetry $f$ such that $f$ is a linear transformation.
An \tbf{affine symmetry} is any symmetry of the form $f(x) = Ax + b$ where $A$ is linear.
\end{defn}

The distinction between linear and affine symmetries is important. In the remainder of this paper, we will assume that any symmetries are linear in this more specific sense unless otherwise specified.

\section{Intractability in the Worst Case}
Before we start discussing how to solve find symmetries in data, let's start by determining how hard it can be. Unfortunately, even if we know the probability distribution explicitly, the answer is that it is very hard in general.

\begin{prop}
\label{intractibility}
The graph automorphism problem on $n$ vertices can be reduced to the decision problem of determining whether a non-trivial linear symmetry exists for a probability measure on $\R^n$.
\end{prop}
\begin{proof}
To prove this, we will encode the graph automorphism problem into the problem of finding a non-trivial linear symmetry of a vector space. Let $\Ga$ a graphs  with $n$ vertices. For simplicity, pick an arbitrary label of the vertices as $\lb i \rb_{i=1}^{n}$. Let $X = \R^{n}$. We define a probability measure $\mu$ on $X$ encoding the graph.

 $\mu = \frac{\sum\limits_{i=1}^{n} \delta_{e_i} + \sum\limits_{(i,j) \in edges(\Ga)} 2 * \delta_{e_i + e_j}}{n + 2 |edges(\Ga)|}$. Where if $v$ is a vector, $\delta_{v}$ is the Dirac delta measure on $v$.

Since the span of points with non-zero measure is all of $X$, we know that any linear symmetry must be surjective, and hence also injective. This also forces points with non-zero measure to be sent to points with equal measure. Therefore, any non-trival linear symmetry will send unit directions to unit directions, thus corresponds to an permutation of the graph vertices. Moreover, if $(i,j) \in edges(\Ga)$ the symmetry must send $e_i + e_j$ to $e_k + e_\ell$, where $(k, \ell)$ is some other edge in $\Ga$. Therefore, the above permutation of graph vertices also respects edges, and hence is an automorphism.
\end{proof}

Note: this problem is not unique to discrete measures. One could find a similar encoding by swapping out the Dirac delta measures with suitable continuous measures.

\section{Symmetries in Distributions}
To start looking for symmetries in a set of data $\cD$, we will want to find statistics which will reflect the overall symmetries in the data. The most direct approach would be to use $\cD$ to construct an approximation $\hat{q}$ of a distribution $q$, and then look for symmetries in $\hat{q}$. Unfortunately, finding a sufficiently precise approximation of a distribution in large dimensions is often infeasible unless you are in a very restrictive setting and have a lot of data. 

Instead, we will look for statistics which summarize the data to a lesser extent, but will still possess the symmetry from the general data. The most straightforward example would be the cumulants. Many linear transformations can fix the mean, so let's consider the first two cumulants: the mean and covariance matrix. For the remainder of this section, we will consider the problem of identifying symmetries in the full distribution, and then later we will discuss how this translates to a sample. In particular, we will walk a fine line of only relying on quantities that can be approximated with a sample.

\subsection{Identifying Symmetries in the Covariance Matrix}

\begin{prop}
If $q$ is a distribution with linear symmetry $A$ then:
\begin{enumerate}
\item $A(\mu) = \mu$ for $\mu = \E_q[X]$
\item $A \Sig A^T = \Sig$ where $\Sig$ is the covariance matrix of $q$
\end{enumerate}
\end{prop}
\begin{proof}
Follows from the linearity of expected values, and a straightforward calculation.
\end{proof}

Since the sample mean and covariance will approach the distribution mean and covariance, these symmetries should hold approximately for the sample statistics.

We now wish to make the second main restriction in this paper. 
\begin{assm}
\label{normal}
$A$ is a normal, linear symmetry of finite order.
\end{assm}

This is actually more restrictive. Since we assume the symmetry has finite order, the eigenvalues must be roots of unity, and hence the symmetry will be orthogonal. Since common known symmetries like image flips and rotations are orthogonal, this is still a large class of symmetries. Moreover, we still haven't escaped the setting where the intractibility of proposition \ref{intractibility} applies. 

However, the results of this assumption is that the action on the covariance matrix becomes an action by conjugation, ie $A \Sig A^T = A \Sig A^{-1} =  \Sig$. This means that $A$ and $\Sig$ must commute. Therefore, $A$ must send each eigenspaces of $\Sig$ onto itself. 

\begin{prop}
\label{involution}
Let assumption \ref{normal} be true. If $\Sig$ exists and has distinct eigenvalues, then $A$ is an involution.
\end{prop}
\begin{proof}
By contradiction, assume that $\Sig$ has unique eigenvalues but $A$ is not an involution. Thus, $A$ must have complex non-real eigenvalues (in particular roots of unity for $n>2$). Since $\Sig$ has unique eigenvalues, the eigenvectors of $\Sig$ are also eigenvectors of $A$. Let $v_\la$ be an eigenvector of $\Sig$ which is in a non-real eigenspace $\la$ of $A$. Since $\Sig$ is symmetric, its eigenvectors are real, so $v_\la$ is a real vector, but $Av_\la = \la v_\la$ is not. However, since $A$ is a real matrix, this is a contradiction.
\end{proof}

\begin{coro}
Under the premises of proposition \ref{involution}, any finite symmetry group $G$ of the distribution $q$ is isomorphic to $(\Z/2\Z)^n$
\end{coro}
\begin{proof}
We know all elements of $G$ must be of order 2. But then $e = (ab)^2 = abab$ for all $a,b \in G$. Multiplying by $ba$ on both sides, we get $ab=ba$, so the group is Abelian, and the result then follows from the fundamental theorem of finitely generated Abelian groups.
\end{proof}

By proposition \ref{involution}, we can relatively easily identify which sets of data may higher order symmetries by looking at the eigenvalues of the covariance matrix. Moreover, if the covariance matrix has degenerate eigenvalues, then learning an involutional symmetry will be significantly harder. So in order to escape proposition \ref{intractibility}, we will make the following assumption:

\begin{assm}
\label{distinct}
The distribution generating $\cD$ has finite covariance with distinct eigenvalues
\end{assm}

This now forces the eigenvectors of $\Sig$ to be eigenvectors of $A$, and we know the eigenvalues of $A$ must be $\pm 1$, so if we know $\Sig$ or its eigenvectors, we immediately get the following:

\begin{prop}
\label{decomp}
Let assumptions \ref{normal} and \ref{distinct} be true.  If $\Sig = V \La V^T$ is an orthogonal eigenvalue decomposition of $\Sig$, then $A = V D V^T$, where $D$ is a diagonal matrix with diagonal entries $\pm 1$ 
\end{prop}
\begin{proof}
See above discussion.
\end{proof}

\subsection{Finding Symmetric Dimensions}

Unfortunately, Assumption \ref{distinct} is insufficient to give us exact results. In fact, every matrix of the form $V D V^T$ will commute with $\Sig$, so $\Sig$ doesn't contain any more information for this purpose. Another way to see this is a central Gaussian in $d$ dimensions has a linear symmetry group $(\Z/2\Z)^d$ which acts by reflecting each principal axis. 

Most distributions don't have as many innate symmetries as a Gaussian, so we shouldn't expect all of these symmetries to also be symmetries of the data we encounter in practice. To see which of these symmetries are real, let's consider other statistics starting with the mean. 

\begin{prop}
Under the assumptions of proposition \ref{decomp}, let $v_i$ be the eigenvectors of $\Sig$. These form a basis. Let $\mu = a_1 v_1 + ... + a_d v_d$ be it's the unique decomposition. If $A v_k = -v_k$ for some $k$, then $a_k = 0$.
\end{prop}
\begin{proof}
$A\mu = b_1 v_1 + ... + b_d v_d$ where $b_i = \pm a_i$. Basis decompositions are unique, so $a_i = b_i$. Thus, $A(v_k) = v_k$ or $a_i = 0$. 
\end{proof}

This is where we needed $A$ to be linear. if $A$ is allowed to be affine with fixed center, then we would need to first shift the mean before looking for zeros. If we don't know the center, then the zeros could be everything or nothing depending on the offset. Therefore, this approach should only be used when the $0$ point of our data is meaningful.

In particular, let's consider a Bayesian approach perspective, and condition on the covariance. As long as the prior distribution on the mean doesn't have any innate discrete concentrations at $0$, we should expect any zeros in the decomposition of $\mu$ come from a symmetry.

We can take alternative approach related to the median instead of the mean, and in practice, the corresponding sample approach is more robust. 
\begin{prop}
Under the assumptions of proposition \ref{decomp}, let $v_i$ be the eigenvectors of $\Sig$. If $A v_k = -v_k$ for some $k$, then $median(Proj_{v_k}(X))= 0$ or equivalently 

$Prob(Proj_{v_k}(X)<0) = Prob(Proj_{v_k}(X)>0)$.
\end{prop}
\begin{proof}
$Proj_{v_k}(A(a_1v_1 + ... + a_n v_n)) = -a_k = -Proj_{v_k}(a_1 v_1 + ... + a_n v_n)$, so the projected distribution is symmetric about $0$
\end{proof}

The mean and median approaches have the advantage of simplicity, but they also have two main weaknesses. First, it restricts us to finding linear transformations, and second it may cause a false positive if the mean happens to be zero in an eigenvector direction for other reasons. 

The first disadvantage can be solved by instead looking at measures of skewness. These will allow us to find affine symmetries. The two simplest choices would be the 3rd cumulent, and the non-parametric skewness. If the projected distribution has 0 skew, then that is strong evidence that there is really a symmetry here, and we can find the offset by looking at either the mean or median.

However, to be sure that at least the distribution on $Proj_{v_k}(X)$ is symmetric, we can use the distance skewness. If the distance skewness is $0$, then the projected distribution must be symmetric, thus avoiding both weaknesses. However, it's still possible that the symmetry in the projection doesn't come from a symmetry in the full distribution, although that would be quite the coincidence.

\begin{prop}
Under the assumptions of proposition \ref{decomp}, let $v_i$ be the eigenvectors of $\Sig$. If $A v_k = -v_k$ for some $k$, then the skewness, non-parametric skewness, and distance skewness (around $0$) of $Proj_{v_k}(X)$ are all $0$
\end{prop}
\begin{proof}
Again follows from the projected distribution being symmetric around $0$
\end{proof}

To confirm that a symmetry is a true symmetry with complete certainty, we could use a non-degenerate Maximum Mean Discrepency or a KS-test, but the key issue here is the curse of dimensionality. In particular, let's define an \tbf{unfixed vector} $v$ to be an eigenvector of $\Sig$ such that $\exists g \in G$ with the property that $gv = -v$. The above propositions give approaches to determine which vectors may be unfixed vectors.

\subsection{Distributions with Multiple Non-trivial Symmetries}
Previously, we identified how symmetries affected a variety of statistics, and how to use these to infer the unfixed vectors. If there is only a single non-trivial symmetry in the distribution, then we can simply negate all the unfixed vectors and combine them into a transformation via $V D V^T$. However, if we don't know how many symmetries the distribution has, you could just as easily have a symmetry group which negates each unfixed vectors independently, or anything in between. In particular, the negation of all unfixed vectors need not be a symmetry. Therefore, we need a way to distinguish between these possibilities. Below we provide one possible approach.

Let $G$ be the complete group of linear symmetries of the distribution. The key is that we have an action on each principal axis $v_\la$, $g v_\la = \pm v_\la$ for all $g \in G$. In particular, if $v_\la$ is an unfixed vector, then $gv_\la = -v_\la$ for some $g$, and so $\lb -v_\la, v_\la \rb$ is an orbit under $G$. Then by the orbit-stabilizer theorem, we know that $|stab_G(v_\la)| = |G|/2$. 
One approach to computing $stab_G(v_\la)$ is to use $stab_G(v_\la) = Aut(\lb x | \lan v_\la,x\ran  > 0\rb )$. In particular, we can tell that $Aut(\lb x | \lan v_\la,x\ran  < 0\rb )$ is trivial if there are no unfixed vectors. Therefore, we can recursively compute $G$.

To be explicit, let's first define the set $X_{V}^+:= Aut(\lb x | \lan v ,x \ran  > 0| \forall v \in V \rb)$. The algorithm goes as follows: we have two variables: $S = \lb v_\la \rb_{\la \in \La}$ which is a set of unfixed eigenvectors, and a sequence of eigenvectors $F = (v_{\la_i})_i$ keeping track of which vectors we've fixed. First, compute the unfixed vectors of $X$, and place them in $S$. Second, take the largest $\la$ corresponding to an unfixed vector which we'll denote by $\la_0$, and append $v_{\la_0}$ to $F$. Then replace $S$ with the unfixed vectors of $X_F^+$ which are also unfixed vectors of $X$. Repeat this process until $S = \emptyset$.

From this, if $n$ is the length of $S$, we can conclude that $|G| = 2^n$ (and thus in particular $G \cong (\Z/2\Z)^n)$). Moreover, the last non-empty set of unfixed vectors, $S_{n-1}$, gives us a true (minimal) symmetry $A_0$ by negating all vectors in $S_{n-1}$ while fixing the rest of the basis.

\begin{lemma}
\label{downward}
At each step in the above algorithm, $G(X^+_{F}) = X$ and any symmetry $h$ of $X^+_{F}$ which commutes with $G$ is an element of $G$.
\end{lemma}
\begin{proof}
By induction, the basis is clear since $G(X) = X$, and a symmetry of $X$ is an element of $G$ by definition. Next, assume that $G(X^+_{F_k}) = X$ and each symmetry of $X^+_{F_k}$ is in $G$. By construction $v_{\la_{k+1}}$ is an unfixed vector, so by the inductive hypothesis there exists a $g \in G$ such that $gv_{\la_{k+1}} = - v_{\la_{k+1}}$ but also by construction $gv_{\la_{i}} = v_{\la_{i}}$ for all $i<k$. Therefore, $gX^+_{F_{k+1}} = X^+_{F_{k}} \cap X^-_{\lb v_{\la_{k+1}} \rb}$, and similarly for combinations of signs of $F_{k+1}$, and hence $G(X^+_{F_{k+1}})=X$. Now let $h$ be a symmetry of $X^+_{F_{k+1}}$. Consider the summing operator $sum_G := \sum_{g \in G} g_*$. We get the following:
\begin{equation}
2^m \mu = sum_G \mu \downharpoonright \ X^+_{F_{k+1}} = sum_G h_* \mu \downharpoonright X^+_{F_{k+1}} = h_* sum_G \mu \downharpoonright X^+_{F_{k+1}} = 2^m h_* \mu
\end{equation}
Where $\mu$ is the distribution on $X$, and $2^m$ is the order of the stabilizer of $X^+_{F_{k+1}}$ in $G$. Thus $h_*$ respects $\mu$, and so by definition is an element of $G$.
\end{proof}

In particular, this implies that $A_0$ is an element of $G$, and $G(X_{F_n}) = X$. In fact, we can get a full generating set as follows. Let $F_{-j} = (v_{\la_i})_{i\neq j}$ be the subsequence of a $F$ where the $j$th element is omitted. Since $F_n \sbs F_{-j}$, we know that $G(X_{F_{-j}}) = X$.

\begin{prop}
The set of unfixed vectors of $X_{F_{-j}}^+$ which are unfixed vectors of $X$ gives a non-trivial element $g_j \in G$, by negating each unfixed vector, fixing each fixed eigenvector of $\Sig$, and then extending by linearity. Together these form a generating set of the group of symmetries $G$.
\end{prop}
\begin{proof}
First let's prove that the $g_j$ are elements of $G$. By insisting that we only consider unfixed vectors of $X$ we force the $g_j$ to commute with $G$. Then the result follows by the same summing argument as Lemma \ref{downward}.

Finally we will prove that $g_j(v_{\la_j}) = -v_{\la_j}$, and since we know $g_i(v_{\la_j}) = v_{\la_j}$ for all $i \neq j$ by construction, this action induces a surjection onto $(\Z/2/Z)^n$, which must be an isomorphism by the equality of cardinality. By strong induction, first we look at $g_n$. We already know that it negates $v_{\la_n}$ by construction. Next, assume that we know $g_{k+1}v_{\la_{i}} = -v_{\la_{i}}$ for all $i > k$. We also know by construction that there exists an element $h_k \in Stab(X_{S_k}^+)$ such that $h_k v_{\la_{k}} = -v_{\la_{k}}$. Unfortunately, it may not fix each $v_{\la_{i}}$ for $i>k$. However, for each $i>k$ such that $h_k v_{\la_i} = - v_{\la_i}$, we can multiply $h_k$ by $g_i$ to prevent this. Assume $\lb \ell_k \rb_k$ is the set of all $\ell>k$ such that $h_k v_{\la_\ell} = - v_{\la_\ell}$. Letting $g_k' = g_{i_1} \cdots g_{i_\ell} h_k$ then produces an element of $G$ which stabilizes $X_{F_{-j}}^+$ while negating $v_{\la_j}$, which implies that $v_{\la_j}$ is unfixed vector of $X_{F_{-j}}^+$, and hence $g_k$ negates it by construction.
\end{proof}

In practice the main issue with this algorithm is that we have reduced the size of our space by a factor of $|G|$. Therefore, for large $|G|$, an alternative way of determining the elements of $G$ which fix $F_k$ may be needed.

\subsection{Non-Orthogonal Symmetries}
The case of learning a symmetry that isn't orthogonal is a bit more tricky. The fundamental problem is that the covariance matrix lives in $X \ot X$, and not in $End(X) \cong X \ot X^*$, so in reality the action on the covariance is $A \cdot \Sig = A \Sig A^T$, and not $\Sig = A \Sig A^{-1}$. The former action doesn't necessarily respect eigenvectors. On the other hand, if there is an invariant inner product, we get an isomorphism $\phi: X \xrightarrow{\sim} X^*$ which is respected by the symmetry $A$. From this we get $(id \ot \phi) (\Sig) \in End(X)$, where the action by $A$ is $A \cdot(id \ot \phi)(\Sig) = A(id \ot \phi)(\Sig)A^{-1}$, and therefore the eigenvectors of $(id \ot \phi)(\Sig)$ are eigenvectors of $A$. Note that we could just have easily asserted that $A$ was symplectic or really that it respects any known non-degenerate form.

If we don't know of such a form, then $\Sig$ can most naturally be thought of as a (positive semi-definite) symmetric form. Without any additional information, the most we can then say about $A$ is that it is orthogonal with respect to this form, which means we still have a continuous space of possibilities.

Luckily, there is a common setting where this issue can be resolved. Assume we have supervised data, with a non-trivial discrete labeling, for example labeled images. Assume moreover the labeling is invariant under the action of $A$. Then we know that $A$ must respect the covariance matrix of each of the labeled subdistributions $p(X|y=c)$ separately. Thus, if we assume one of the subdistributions has a non-degenerate covariance matrix, we can rewrite the space $X$ in terms of its eigenvector basis (or the eigenvectors for the combined dataset), and then use the previous methods on the other subdistributions. Of course, all of this assumes the existence of a symmetry fixing multiple subdistributions simultaneously, which is going to be rarer than a symmetry fixing a single subdistribution.

\subsection{Labeled Distributions and Subrepresentations}
If we do have access to some invariant labels in a dataset but also know that the symmetry is orthogonal, then the problem becomes somewhat easier. Now, we have a set of covariance matrices $\Sig_c$ for each label $c$, and the symmetry transformation $A$ must commute with all of these matrices. In mathematical terms, $A$ must be an intertwiner for the defining representation of the algebra $\lan \Sig_c \ran_{c}$. If $A$ is an involution, such as is the case if we take assumption \ref{distinct} that $\Sig$ have distinct eigenvalues, this implies that the $+1$ and $-1$ eigenspaces of $A$ must be subrepresentations, and in particular that $\lan \Sig_c \ran_{c}$ is decomposable.

In more concrete terms, there is a basis $\lb v_i \rb_i$ such that the set of matrices $\lb \Sig_c \rb_c$ is simultaneously block diagonal. These subspaces then take the place of the eigenvectors in the non-labeled case. We just need to select some number of these subspaces to be negated in order to get a matrix which respects the covariances. Generically, we should expect that as long as the number of labels is sufficiently large (maybe equal to or greater than $4$ cf. \citep{akm:numgenerators}), then any decomposition $\lb \Sig_c \rb_c$ should come from the symmetry. In particular, if there is only a single  non-trivial symmetry, then there should be only $2$ subrepresentations to choose from.

Unfortunately, the noise coming from the sampling requires that our method of detecting subrepresentations needs to be robust. Without distinct eigenvalues, we might need to use a more complicated approach as in \citep{mm:errorcontrolled}, but this is problematic because it requires finding the eigenvectors of a $d^4$ dimensional matrix. As a more direct approach, we can diagonalize the total covariance matrix $\Sig$ (including the non-labeled data), then rewrite each $\Sig_c$ in the corresponding basis, and then in this basis each $\Sig_c$ should be almost block diagonal. In order to control for different variances, we'll use the correlation instead. We can write $\tilde{P} = \frac{1}{C}\sum_c abs(P_c)$ where $abs(P_c)$ is the pointwise absolute value of the correlation matrix associated to $c$ in the basis which diagonalizes $\Sig$ (alternatively, we can make this distribution closer to a normal distribution by taking the average of $\tilde{\tht}_c = \frac{1}{2}(\log(1+ abs(P_c)) - \log(1 - abs(P_c))$ instead). To find the block diagonal structure, we can interpret this $d \times d$ matrix as a similarity matrix, and then do clustering in order to determine the blocks. Moreover, we can then improve the clustering and determine which eigenvectors are fixed by adding in an auxiliary point $p_{aux}$ whose similarity to each eigenvector is equal to the average of normalized absolute means: $similarity(p_{aux}, p_\la) = \frac{1}{C} \sum_c \frac{\lvert \lan \mu_c, v_\la \ran \rvert}{\sqrt{\la}}$.

Note: the best (similarity-based) clustering algorithm will depend on the the size of $d$. For small $d$, we can directly find an optimal normalized cut. In the case where this becomes intractable, we need to be careful about the potential for eigenvectors with large error. These eigenvectors will tend to have a large affinity to $2$ or more clusters, making it hard for certain algorithms to split those clusters apart. Methods such as $k$-medoids or affinity propagation may select these erroneous eigenvectors as exemplars. 

As a simple example of a fast algorithm which is robust to the above issue, we can take a modified single-step version of dissimilarity analysis \citep{mwdm:dissim} by first making two clusters with the auxiliary point as one cluster, and all other points in the other. Then we iteratively add points to the auxiliary-point cluster such $d_\la := d_{avg}(p_\la, \text{auxiliary cluster}) - d_{avg}(p_\la, \text{swap cluster})$ is maximal, until $d_\la$ is negative for all $\la$, or a desired number of vectors are selected. If we included the auxiliary node, then this algorithm can be interpreted as an extension of the unsupervised mean-based ranking method by incorporating label-specific covariance data. 

\section{Symmetries in Samples}
The previous section looked at the case of finding symmetries assuming we know the true distribution and can compute expected statistics. Now we move to the more practical setting where we don't know the true distribution, and only have a sample which we'll denote $\cD$. However, everything works more-or-less the same asymptotically except using sample statistics. Assuming the sample size is large, we know the sample mean and sample covariance will become arbitrarily close to the distribution mean, median and covariance. Therefore, assuming distinct eigenvalues of the covariance, the sample eigenvectors should converge to true eigenvectors. Therefore, as before we can find a (now approximate) symmetry by negating the eigenvectors such that the corresponding direction has mean, median, or skew that is \tit{close enough} to $0$. If the symmetry is actually a symmetry of the sample, then these statistics will be $0$ on the nose, but otherwise we need to develop a test for when they are \tit{close enough} to $0$.

The key challenge is that we have two points of approximation. First, we're approximating the covariance eigenvectors, and second we are approximating some other statistic which dependent on that eigenvector to determine if it is likely an unfixed vectors. One approach would be to ignore the uncertainty in the eigenvectors, and then look for the dimensions where the confidence interval of the desired statistic overlaps with $0$. Unfortunately, this works poorly in tests with synthetic data.

\subsection{Ranking Approaches}
\label{rankings}
A more effective approach is to rank the dimensions by how appropriate it would be to negate them. When using the mean, it is possible to just use the magnitude of a eigenvectors coefficient in the decomposition $\E[X] = a_1 v_{\la_1} + ... + a_d v_{\la_d}$. The advantage of this is that quadratic mean of the $a_i$'s for negated $v_{\la_i}$ is exactly the distance that the mean is moved by the linear transformation. Moreover, as the sample size grows, these $a_i$ will converge to $0$. Unfortunately, this doesn't account for the fact that some dimensions have larger covariance, so using the raw $a_i$ will tend to select transformations that don't change datapoints much. Therefore, to correct for this, we can divide out the standard deviation, giving $\lb a_i/\sqrt{\la_i} \rb_{i=1}^d$. 

Similarly, we can use other statistics like the the (normalized) median, a sign test statistic, or the skew. In experiments with synthetic (Gumbel mixture model) data, the mean and median seem to be the most effective. In fact, a mixture of the two seems to work best. This makes sense because the skew requires more data to converge, while the sign statistic will tend to be misled by mixture data. 

Once we have an effective ranking, we have reduced the number of models to check from $2^d$ to $d$, which allows us to take a more global approach to model selection. There are a number of non-parametric approaches to comparing samples. In this case, we will use \tbf{maximum mean discrepancy} (MMD)\citep{ss:kernellearning} with an Squared Exponential kernel $k(x,x') = e^{\lVert x - x' \rVert^2}$. This has the two main advantages that it is zero only when the distributions are equal, and it is differentiable. We will use the differentiability later in order to fine-tune the transformation. As a variant, we can use a weighted $L^2$ norm to greater emphasize errors coming from lower variance dimensions. In particular, when doing model selection with MNIST data we will use a regularized matrix of the form $\Om' = ((1-\al) \Sig + \al I)^{-1}/h$, giving us $k_{\Om'}(x,x') = e^{-(x - x')^T\Om' (x - x')/2}$.

After fixing a scoring statistic, we then have a couple approaches to applying it. The most direct approach is to fit a separate transformation $T_i$ on all the data for each choice $i$ of the number of negative eigenvalues. Then we can pick the $i$ which minimizes $error(X, X * T)$.

The second approach is to use $k$-fold cross-validated scoring. We split the data into k disjoint folds $\lb X_i \rb$, then fit the linear transformation $T_i$ on the union of $k-1$ folds. Finally we can compute $error(X_i, X_i * T_i)$, and average these over $i = 1, 2, ..., k$. This has the advantage of reducing the bias in the scoring due to fitting process, and so the results should be more generalizable. It has the disadvantage of being considerably slower, and we don't necessarily have a guarantee that the model we finally fit works well for the full dataset. 

There are a few important practical details to note:
\begin{enumerate}
\item The iterative approach to approximating the MMD tends to have a significant variance. Because of this, in practice it is important to use repeated k-folds. \item If the ordering of eigenvectors is incorrect, this approach has a tendency to select only a single negative eigenvector. In order to select a decent (but necessarily imperfect) transformation in this case, we can use the one standard error rule to select the most negative eigenvectors whose error is less than a standard error above the minimum error. 
\item If there are more than one non-trivial symmetry, some certain combinations of unfixed vectors may be incompatible. The simplest example is 
\begin{equation*}
G = \lb I_3,  \: diag(-1, -1, 1), \: diag(-1, 1, -1), \: diag(1, -1, -1) \rb
\end{equation*}
where $diag(-1, -1, -1)$ is not a symmetry. In this case, we may be unable to get a complete symmetry using this method, so alternative methods would be necessary.
\end{enumerate}

\subsubsection{CLT-Based Bayesian Model Selection}
\label{bayes_model}

A much faster, but biased approach is to use (partial) Bayesian Model selection on the mean of each dimension. We'll see that we obtain similar results. First, as before, we will ignore the uncertainty of the covariance and assume that the covariance matrix is equal to the sample covariance matrix, and hence so are the eigenvectors and eigenvalues. For each eigenvector $v_{\la_i}$ we define $\cD_i$ to be the data projected onto $v_{\la_i}$, ie $\lb Proj_{v_{\la_i}}(x) | x \in \cD \rb$, or if $X$ is the design matrix then $\cD_i$ has a $1D$ design matrix $X_i = Xv_{\la_i}$. For each $i$, let $m_0$ be a model where $\cD_i$ has an unknown distribution with variance $\la_i= \hat{\la}_i$, which is our model if $v_{\la_i}$ is an unfixed vector. Let $m_1$ be the model that $\cD_i$ be an unknown distribution with mean $\mu_i$ and variance $\la_i = \hat{\la}_i$, where $\mu_i \sim \cN(0, \phi(\la_i))$, for some function $\phi$. This is our model for a fixed vector.

Next we consider the case where we only know the first two moments of the data. By the Central Limit Theorem we can approximate the two top-level distributions of $\hat{\mu}_i$ as a Gaussian with variance $\hat{\la}_i/N$. Thus we get the following approximations:
\begin{align}
\hat{\mu}|\Sig, m_0 &\sim \cN(0, \hat{\la}_i/N) \\
\hat{\mu}|\Sig, m_1 &\sim \cN(\mu_i|0, \ph(\la_i)) + \cN(\ep|0, \hat{\la}_i/N) = \cN(0, \ph(\la_i) + \hat{\la}_i/N)
\end{align}

Which we'll denote as $P_{CLT}$. We want to select the model which makes the observed sample mean $\ol{x}$ most likely.

\begin{prop} If $\phi(\la) = \la$ and $\la_i \neq 0$, then
$P_{CLT}(\hat{\mu}|\Sig, m_0) > P_{CLT}(\hat{\mu}|\Sig, m_1)$ if and only if 
\begin{equation}
\frac{\ol{x}_i}{\sqrt{\hat{\la}_i}} > \frac{\sqrt{(N+1)\ln(N+1)}}{N}
\end{equation}
\end{prop}
\begin{proof}
\begin{align*}
\cN(x|0, \hat{\la}_i/N) = \frac{1}{\sqrt{2 \pi \hat{\la}_i/N}} e^{-\frac{x^2}{2\hat{\la}_i/N}} &< \frac{1}{\sqrt{2 \pi \hat{\la}_i\frac{N+1}{N}}} e^{-\frac{x^2}{2\hat{\la}_i\frac{N+1}{N}}} = \cN(x|0, \hat{\la}_i + \hat{\la}_i/N) \\
&\Longleftrightarrow \\
\sqrt{N+1} e^{-\frac{x^2}{2\hat{\la}_i/N}} &< e^{-\frac{x^2}{2\hat{\la}_i\frac{N+1}{N}}} \\
&\Longleftrightarrow \\
\frac{1}{2}\ln(N+1) - \frac{x^2}{2\hat{\la}_i/N} &< -\frac{x^2}{2\hat{\la}_i\frac{N+1}{N}} \\
&\Longleftrightarrow \\
\ln(N+1) &< \frac{x^2}{\hat{\la}_i}(N - \frac{N}{N+1}) \\
&\Longleftrightarrow \\
\frac{x}{\sqrt{\hat{\la}_i}} &> \frac{\sqrt{(N+1)\ln(N+1)}}{N}
\end{align*}
\end{proof}

Therefore, if we use the prior that the two models are equally common, then we obtain a criterion for determining the cut-off point for the mean-based ranking approach. Notice that we've made a significant assumption that the mean's prior distribution has variance equal to the observed variance. This was chosen primarily to make the math nicer, but asymptotically the results should converge to the same result. 

\subsection{Combined Bootstrapping Approaches}
Another principled, non-parametric approach would be to use end-to-end bootstrapping. Take the null hypothesis to be that the dimension is symmetric, and fix a particular univariate statistic $T$, distance function, and significance $\al$. We'll record data in sets, $S_v$ and $N_v$, for each eigenvector of the original covariance matrix $\hat{\Sig}$. First, resample the data $m$ times, giving new samples $\lb \cD_i \rb_{i=1}^m$. Second, compute the covariance matrices $\hat{\Sig}_i$ for each sample. Third, for each eigenvector $v$ of the original covariance $\hat{\Sig}$, determine the closest eigenvector $v'$ of $\hat{\Sig}_i$. Record that distance in $N_v$. Finally, compute $T$ on the projection of $\cD$ onto $\R v'$, and record that value in $S_v$.

We can use the distances $N_v$ to determine if the eigenvalue for $v$ is likely distinct by comparing them to a uniform distribution on a circle, but for simplicity, let's assume that we know that the true eigenvalues of $\Sig$ are distinct, and the distances are consistently small, so we know that we have a consistent correspondence. Then we can order the values in $S_v$, and use the bootstrap percentile interval (under the null hypothesis this is equivalent to the bootstrap pivotal interval). If the confidence interval contains $0$, then we fail to reject the null hypothesis, and we add $v$ to the potential unfixed vector. Otherwise, we conclude that $v$ is a fixed vector under the action of $G$. 

Unfortunately there is significant variance in the eigenvectors even for relatively large datasets, which can make the closest eigenvector $v'$ a bad approximation, and leading to an inability to reject the null hypothesis, at least without a very large dataset.

\subsection{Theoretical Bounds}
\label{theoretical}

With some fairly direct calculations, we can bound the error of these methods in terms of the error of the covariance matrix and mean, which will give us consistency. For now $\hat{\mu}$ and $\hat{\Sig}$ can be any estimators of $\mu$ and $\Sig$, but we'll be primarily interested in the sample mean and covariance matrix. Let $\lb \la_i \rb_i$ be the eigenvalues of the population covariance $\Sig$, with (some choice of) corresponding unit eigenvectors $\lb v_{\la_i} \rb$,  $\mu_i = \lan \mu, v_{\la_i} \ran$, and $\tilde{\mu}_i = \lan \hat{\mu}, v_{\la_i} \ran$. Note: $\tilde{\mu}_i$ is not a statistic since it depends on $v_{\la_i}$, but importantly its difference with $\mu_i$ can be controlled using the Central Limit Theorem. First, we assume the following bounds:

\begin{align}
\label{mu_bound}
\lvert \tilde{\mu}_i - \mu_i \rvert &< \ep_{1i} \; \forall i \\
\label{sig_bound}
\lVert \hat{\Sig} - \Sig \rVert_{op} &< \ep_2 \\
\label{la_bound}
\lvert \la_i - \la_j \rvert &\geq 2 \delta \; \forall i\neq j
\end{align}

For some $\ep_i>0$ and $\delta > 0$. First, we want to bound the error for the eigenvalues and eigenvectors. Let $\lb \hat{\la}_k \rb_k$, $\hat{v}_k$ be the set of eigenvalues and corresponding eigenvectors for $\hat{\Sig}$. However, a priori these may not be in natural correspondence with the eigenvalues of the population covariance $\Sig$. Fix a particular estimated eigenvalue $\hat{\la}_k$. Since the eigenvalues are distinct by assumption, we can decompose $\hat{v}_k$ in terms of the population eigenvectors:

\begin{equation}
\hat{v}_k = \sum\limits_{i=1}^d a_i v_{\la_i}
\end{equation}

Where we know $\sum\limits_{i=1}^d a_i^2 = 1$, since by assumption $\hat{v}_k$ is normalized, and $v_{\la_i}$ form an orthonormal basis. Then can use the bound on the covariance to bound the difference of eigenvalues:

\begin{align}
\ep^2_2 &> \lVert (\hat{\Sig} - \Sig)\hat{v}_k \rVert_2^2 \\
&=  \lVert \hat{\Sig}v_{\la_k} - \Sig \sum\limits_{i=1}^d a_i v_{\la_i} \rVert_2^2 \\
&= \lVert \hat{\la}_k \hat{v}_k - \sum\limits_{i=1}^d a_i \la_i v_{\la_i} \rVert_2^2 \\
&= \lVert \sum\limits_{i=1}^d (a_i\hat{\la}_k - a_i \la_i) v_{\la_i} \rVert_2^2 \\
\label{ep_2_breakdown}
&= \sum\limits_{i=1}^d a_i^2(\hat{\la}_k - \la_i)^2 
\end{align}

Since the $\la_i$ have minimum pairwise distance $2\delta$, we know that at most one of the $\la_i$ satisfies $\lvert \hat{\la}_k - \la_i \rvert < \delta$. Let's consider the case where none of them satisfies this inequality. Then:

\begin{align}
\ep^2_2 > \sum\limits_{i=1}^d a_i^2(\hat{\la}_k - \la_i)^2 \geq \delta^2 \sum\limits_{i=1}^d a_i^2 = \delta^2
\end{align}

And therefore, by contrapositive, as long as the covariance error is small enough, (ie $\ep_2 < \delta$), there is exactly one such $v_{\la_i}$ near $\hat{v}_k$, which produces a correspondence. So from now on, we'll assume that this inequality holds, and therefore after relabeling the eigenvectors we get that $\lvert \hat{\la}_k - \la_i \rvert < \delta$ if and only if $i = k$. 

Next, we want to bound the distance between $\hat{v}_k$ and $v_{\la_k}$. From the above assumption, we get the following breakdown of (\ref{ep_2_breakdown}):

\begin{align}
\ep_2^2 > \sum\limits_{i=1}^d a_i^2(\hat{\la}_k - \la_i)^2 &= a_k^2(\hat{\la}_k - \la_k)^2 + \sum\limits_{i\neq k} a_i^2(\hat{\la}_k - \la_i)^2 \\
\label{eigenvalue_bound}
&\leq a_k^2(\hat{\la}_k - \la_k)^2 + \delta^2 \sum\limits_{i \neq k} a_i^2
\end{align}

Which in particular gives us: 

\begin{equation}
\sum\limits_{i \neq k} a_i^2 < \frac{\ep_2^2}{\delta^2}
\end{equation}

But using $\sum\limits_{i=1}^d a_i^2 = 1$, we also get a bound on $a_k$:

\begin{align}
a_k^2 &= 1 - \sum\limits_{i \neq k} a_i^2 \\
&> 1 - \frac{\ep_2^2}{\delta^2}
\end{align}

So as $\ep_2 \to 0$, $a_i \to \delta_{ik}$ where $\delta_{ik}$ is the Kronecker delta function, and so $\hat{v}_k \to v_{\la_k}$. In particular, we have:

\begin{align}
\lVert \hat{v}_k - v_{\la_k} \rVert^2_2 &= (1 - a_k)^2 + \sum\limits_{i \neq k} a_i^2 \\
&< \left (1 - \sqrt{1 - \frac{\ep_2^2}{\delta^2}} \right)^2 + \frac{\ep_2^2}{\delta^2} \\
&= 2\left(1- \sqrt{1-\frac{\ep_2^2}{\delta^2}}\right) =: g(\ep_2, \delta) ^ 2
\end{align}

Using the other half of (\ref{eigenvalue_bound}), we are also able to bound the difference of eigenvalues:

\begin{equation}
(\hat{\la}_k - \la_k )^2 < \frac{\ep_2^2}{1- \frac{\ep_2^2}{\delta^2}}\\
\end{equation}

So as $\ep_2 \to 0$ we also get $\hat{\la}_k \to \la_k$. Now let's consider more specifically our topic, and use the mean inequality (\ref{mu_bound}). What we really care about is the estimator $\hat{\mu}_k = \lan \hat{\mu}, \hat{v}_k \ran$ which we use to determine which eigenvectors to negate.

\begin{align}
\lvert \hat{\mu}_k - \mu_k \rvert &= \lvert \sum\limits_{i = 1}^d a_i \lan \hat{\mu} , \hat{v}_i\ran - \mu_k \rvert \\
&= \lvert \sum\limits_{i = 1}^d a_i \lan \hat{\mu} , \hat{v}_i - v_{\la_i}\ran + \sum\limits_{i = 1}^d a_i \lan \hat{\mu} , v_{\la_i}\ran - \mu_k \rvert \\
& \leq \lVert \vec{a} \rVert_2 \sum\limits_{i=1}^d \lVert \hat{\mu} \rVert_2 \lVert \hat{v}_i - v_{\la_i} \rVert_2 + \lvert \sum\limits_{i=1}^d a_i\tilde{\mu}_i - \mu_k \rvert \\
&< d \cdot \lVert \hat{\mu} \rVert_2 g(\ep_2, \delta) +  \lvert \tilde{\mu}_k - \mu_k \rvert + \lvert (a_k-1) \tilde{\mu}_k + \sum\limits_{i \neq k} a_i \tilde{\mu}_i \rvert \\
&< d  \lVert \hat{\mu} \rVert_2 g(\ep_2,\delta) + \ep_{1k} + g(\ep_2, \delta) \sqrt{ \sum\limits_{i=1}^d \tilde{\mu}_i^2} \\
&< (d + 1) g(\ep_2,\delta) \lVert \hat{\mu} \rVert_2 + \ep_{1k}
\end{align}

Putting it all together, we get the following theorem:
\begin{thm}
Assume inequalities (\ref{mu_bound}) - (\ref{la_bound}) hold for some estimators $\hat{\mu}$ and $\hat{\Sig}$, and further $\ep_2 < \delta$, then we get the following error bound:
\begin{equation}
\lvert \hat{\mu}_k - \mu_k \rvert < \ep_{1k} + \sqrt{2}(d + 1)\lVert \hat{\mu} \rVert_2 \sqrt{1- \sqrt{1-\frac{\ep_2^2}{\delta^2}}}
\end{equation}
\end{thm}

In particular, the operator norm is bounded by the Frobenius norm which is $\sqrt{n}$ convergent, and $\sqrt{2}\sqrt{1- \sqrt{1-\frac{\ep_2^2}{\delta^2}}} \to \frac{\ep_2}{\delta}$ as $\ep_2 \to 0$. 

\begin{coro}
Let $\hat{\Sig}$ and $\hat{\mu}$ be the sample covariance matrix and sample mean respectively, and assume inequality (\ref{la_bound}) holds, then $\hat{\mu}_k$ is strongly consistent for $\mu_k$ and if further and the fourth moment is finite, then it is $\sqrt{n}$-consistent.
\end{coro}
\begin{proof}
Since $\hat{\mu}$ converges a.s. by the (strong) law of large numbers, we know $\lVert \hat{\mu} \rVert_2$ is bounded. The former statement then follows by the (strong) law of large numbers applied to $\mu$ and $\Sig$, while the latter follows from the CLT.
\end{proof}

Finally we can use use this to analyze our previous methods: 

\begin{coro}
Make Assumptions \ref{normal} and \ref{distinct} and let $\hat{\Sig}$ and $\hat{\mu}$ be the sample covariance matrix and sample mean respectively. Consider the population property $T_i(P) := \begin{cases} 
      1 & \mu_i=0 \\
      0 & \mu_i \neq 0
\end{cases}$, then any estimator of the form $\hat{T}_i(\cD) = \I(\hat{\mu}_i < a_n)$ where $\frac{\sqrt{n}}{a_n} \to 0$ and $a_n \to 0$ as $n \to \infty$ is consistent.
\end{coro}

In particular, this shows that method in subsection \ref{bayes_model} is consistent assuming that the only $\mu_i$ which are $0$ come from unfixed vectors.

\subsection{Covariance-Adjusted Rankings}
The theoretical results of subsection \ref{theoretical} suggest approach to incorporating the covariance error via the approximate error bound $\ep_{1k} + (d + 1) \lVert \hat{\mu} \rVert_2 \frac{\ep_2}{\delta}$. Unfortunately, in the current form, the adjustment is a constant with respect to $k$. 

To obtain a more useful adjustment, we need to make things more local. To do this, we consider the additional error that comes from each $a_i \neq 0$ for $i \neq k$. Asymptotically, this contributes $\lvert \hat{\mu}_i - \hat{\mu}_k \rvert \lvert a_i \rvert$ worth of error. However, considering the actual problem of finding the symmetry, we are not worried about the case when $\lvert \mu_i \rvert \leq  \lvert \mu_k \rvert = 0$, because then $\lvert \mu_i \rvert$ wouldn't be contributing to an incorrect decision. Therefore, we replace it with $(\lvert \hat{\mu}_i\rvert - \lvert\hat{\mu}_k \rvert)_+$, which asymptotically only adds error when $v_{\la_i}$ is fixed and $v_{\la_k}$ is unfixed. Next, we estimate the sample variance of $a_i$ by $\frac{1}{\lvert \la_k - \la_i \rvert}\frac{Corr(\lan \hat{v}_i, (X-\mu) \ran ^2, \lan \hat{v}_k, (X-\mu) \ran ^2)}{\sqrt{n}}$, giving a ranking statistic:

\begin{equation}
T_k(\cD) = \hat{\mu}_k / \left(\sqrt{\la_k} + \sum\limits_{i \neq k}\frac{(\lvert \hat{\mu}_i\rvert - \lvert\hat{\mu}_k \rvert)_+}{\lvert \la_i - \la_k \rvert}Corr(\lan \hat{v}_i, (X-\hat{\mu}) \ran ^2, \lan \hat{v}_k, (X-\hat{\mu}) \ran ^2)\right)
\end{equation}

In the experiment section we will see that this extra complication can significantly improve performance.

\section{Fine-tuning using a Linear MMD Network}
\label{finetuning-section}
We started by forcing the symmetry to respect the second cumulant (ie the covariance) rather than a more holistic view including the first cumulant (the mean). This is because the mean has insufficient information to narrow down our search to finitely many transformations. However, this approach is somewhat myopic and the eigenvectors found may have some error that could be reduced by including more information. To address this, we can instead try to learn a transformation that minimizes the MMD (Maximum Mean Discrepancy) for a non-degenerate kernel (we'll use an Squared Exponential Kernel). This is in some sense a natural generalization of the previous approach since training using the MMD corresponds to moment matching after embedding into a Hilbert space. Moreover, unlike the previous moment matching 

Assume we have a decomposed approximate transformation $A = V D V^T$ using the previous approach, so that $V$ is the space of eigenvectors of the covariance, and $D$ is a diagonal matrix with entries $\pm 1$. We consider a transformation $A_W = W D W^T$ with matrix parameter $W$, where $D$ is fixed. Initialize $W$ at $V$, and minimize the MMD between $A_W X$ and $X$. Since the MMD is differentiable, this can be approached using a standard gradient descent-based methods. We used Stochastic Gradient Descent with a momentum of $0.5$. We also make a slight modification of the sample MMD where we omit the contribution of point $x$ with its transformed self $A_W x$ in order avoid biasing the algorithm towards learning an identity map. 

\begin{remark}
One might wonder why we didn't just use the MDD network approach to begin with if it is less myopic. The key issue is that this model's objective function is highly multimodal. This seems to correspond to the fact that if you have a $W$ which respects the covariance, then there are $2^n$ discrete local optima corresponding to choosing which eigenvectors to negate. To travel between this modes, the transformation would need to stop respecting some of the eigenvectors, which makes these modes rather deep. Likely because of this, in practice with a random initialization, the algorithm tends to get stuck in poor local optima.
\end{remark}

\section{Experimental Results}
In order to determine which approaches are optimal, and determine the applicability of these methods we'll use two types of experimentation. First, we'll consider synthetic data where we can compute the ground truth error under certain conditions. Then we'll consider the more realistic MNIST dataset, and see how well these techniques might work in practice.

\subsection{Synthetic Data}
In order to make sure the assumptions of the paper hold, we can repeatedly build synethetic datasets and use these approaches to test them. In order to avoid models with additional symmetries, we used a mixture of skew base distributions, in this case Gumbel distributions. We differentiate each Gumbel distribution in the mixture by multiplying it by a random invertible matrix, and shifting its mean by vector produced component-wise by a normal distribution, truncated to stay within $\pm 2\sig$. The invertible matrix is produced as the product of a random strictly lower triangular matrix and upper triangular matrix, where the diagonal is produced in a more involved way in order to keep things more stable. We first take $d$ samples of a standard normal distribution truncated between $0.4$ and $2$, and divided by $4$, making all values $>0.1$. Then we produce a new array so that the $i$th entry is the sum of the first $i$ random samples. Finally, we permute these values randomly, and set the diagonal of the upper triangular matrix to be equal to this array. This ensures that none of the diagonals are within $0.1$ of each other.

To produce a sample with symmetry, we first fix a particular linear transformation so that the results can be compared. In this case, we used the transformation which swaps pairs of coordinate vectors $e_0 \leftrightarrow e_1$,  $e_2 \leftrightarrow e_3$, etc... Then, we take $N/(2 * num\_clusters)$ samples from one of the Gumbel distributions, add them to our dataset, sample another $N/(2 * num\_clusters)$ samples from that same distribution, but then also multiply them by the fixed symmetry. Then we repeat this for each cluster. Finally, we divide the whole dataset by the scalar standard deviation for stability and hyperparamter consistency. In the following, we will use only $2$ clusters.

\subsubsection{Dataset Parameters and Ranking Methods}
First let's compare the different ranking methods. We'll keep the model selection method fixed by using the true number of swapped eigenvectors, and also keep the computational cost down by not fine tuning. We measure the error as the $L^2$ norm of the difference of the predicted and ground truth matrices. Fixing the number of dimensions to be $10$, and using $1,000$ datasets, we get the following table containing the MMD and standard errors. 

\begin{table}[H]
\centering
\begin{tabular}{l|l|l|l|l|l}
Samples & Mean            & Median          & MM Mix & Sign            & Skew            \\\hline
2,000                         & 0.239 $\pm$ 0.003 & 0.236 $\pm$ 0.003 & 0.227 $\pm$ 0.003            & 0.270 $\pm$ 0.003 & 0.326 $\pm$ 0.003 \\
10,000                        & 0.168 $\pm$ 0.004 & 0.167 $\pm$ 0.004 & 0.151 $\pm$ 0.004            & 0.211 $\pm$ 0.004 & 0.247 $\pm$ 0.004 \\
50,000                        & 0.111 $\pm$ 0.004 & 0.113 $\pm$ 0.004 & 0.097 $\pm$ 0.004            & 0.157 $\pm$ 0.004 & 0.173 $\pm$ 0.004 \\
250,000            & 0.066 $\pm$ 0.003 & 0.064 $\pm$ 0.003 & 0.052 $\pm$ 0.003 &           0.111 $\pm$ 0.004 & 0.104 $\pm$ 0.004
\end{tabular}
\end{table}

As we can see, an even mixture of Mean and Median seems to be the most effective approach of this set, although the difference with the mean and median isn't particularly large. Moreover, increasing the sample size by a factor of $5$ seems to correspond to roughly a linear decrease in error on order of $0.05$. On the other hand, the sign and skew approaches tend to lag in error, with the skew starting with larger error and but eventually overtaking the sign statistic.

However, all of these methods only look at the raw statistics without considering the error in the covariance eigenvectors. Using the mean, but adjusting for the error in the covariance eigenvalues we get a significantly improved results which are unfortunately a bit slower to calculate. We can compare them to the unadjusted best results below:
\begin{table}[H]
\centering
\begin{tabular}{l|l|l|l|l|l}
Samples  			  & MM Mix & Cov Adjusted \\\hline
2,000                  & 0.227 $\pm$ 0.003 & 0.193 $\pm$ 0.003  \\
10,000                 & 0.151 $\pm$ 0.004 & 0.109 $\pm$ 0.003  \\
50,000                 & 0.097 $\pm$ 0.004 & 0.066 $\pm$ 0.003  \\
250,000                & 0.052 $\pm$ 0.003 & 0.032 $\pm$ 0.002 
\end{tabular}
\end{table}

Looking just at the case of 50,000 samples and a Mean-Median mix, we can plot the histogram of errors giving Figure \ref{dim10}. There are two clear modes, one corresponding to being near the global minimum, while the other seems to correspond to selecting two of the eigenvectors incorrectly. Around 69\% of the tests avoid an incorrect selection, and therefore should be be near the global optima.

\begin{figure}[H]
\centerline{\includegraphics[scale=1.5]{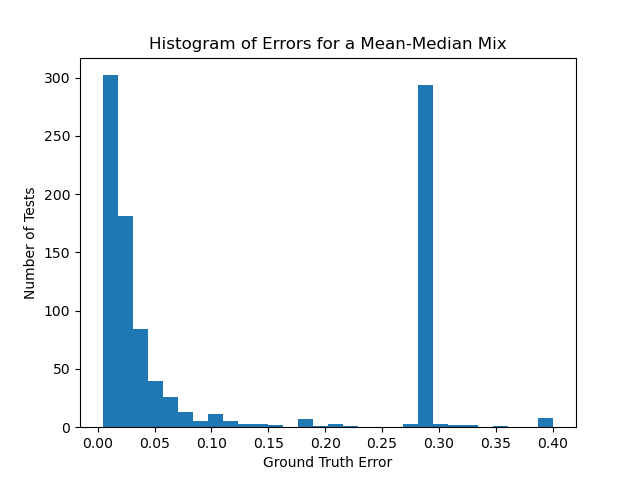}}
\caption{Error histogram for $N=50,000$ and $d=10$}
\label{dim10}
\end{figure} 

The effect of increasing the dimension is a bit more subtle. If the dimensions is low enough, this approach will usually find a transformation near the global optima, and so has a small error. As the dimension increases, the probability of choosing all the eigenvectors correctly significantly decreases, as can be seen in Figure \ref{dim22}.

\begin{figure}[H]
\centerline{\includegraphics[scale=0.48]{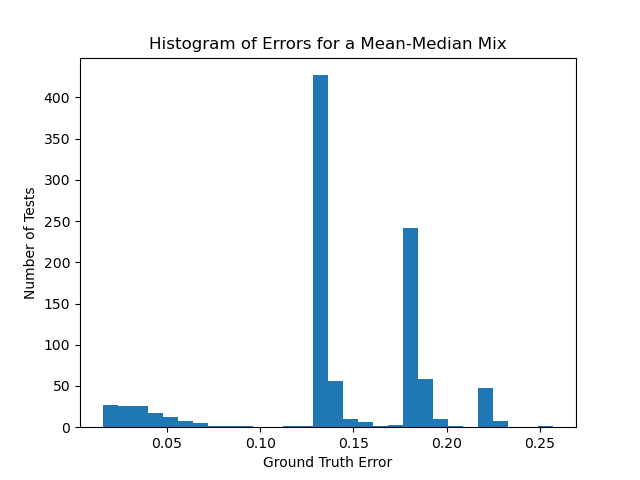}}
\caption{Error histogram for $N=50,000$ and $d=22$. Compare to Figure \ref{dim10}.}
\label{dim22}
\end{figure} 

However, even though the chance of being near the global optima decreases, the error eventually stabilizes as can be seen in the first two rows of the following table, which records the results of different dimensions and sample sizes using the covariant adjusted statistic (even clearer results happen for other statistics). This suggests that the bound in Theorem \ref{theoretical} which has $d$ dependence isn't sharp with respect to $d$ complexity, at least at low dimensions. For each pair of parameters, we produced $1,000$ synthetic datasets, and tested the algorithm on each, recording the mean of ground truth errors $\pm$ the standard errors.

\begin{table}[H]
\centering
\begin{tabular}{l|l|l|l|l|l}
N\textbackslash{}d & 6                 & 10                & 14                & 18                & 22                \\ \hline
2,000              & 0.127 $\pm$ 0.005 & 0.190 $\pm$ 0.003  & 0.212 $\pm$ 0.002 & 0.211 $\pm$ 0.001 & 0.208 $\pm$ 0.001 \\
10,000             & 0.064 $\pm$ 0.003 & 0.113 $\pm$ 0.003 & 0.146 $\pm$ 0.003 & 0.153 $\pm$ 0.002 & 0.157 $\pm$ 0.001 \\
50,000             & 0.033 $\pm$ 0.002 & 0.065 $\pm$ 0.003 & 0.087 $\pm$ 0.002 & 0.097 $\pm$ 0.002 & 0.110 $\pm$ 0.002  \\
250,000            & 0.017 $\pm$ 0.002 & 0.031 $\pm$ 0.002 & 0.045 $\pm$ 0.002 & 0.058 $\pm$ 0.002 & 0.069 $\pm$ 0.002
\end{tabular}
\end{table}

From this we can see that each $5$-fold increase in the sample size leads to a reduction of mean error of around $0.04-0.05$ for larger dimensions in this synthetic dataset model. We can also see that a large dataset in the hundreds of thousands of sample points may be necessary to produce very low error. Luckily, this is feasible for unsupervised image datasets, as long as standard incremental methods are used to compute the desired statistics. In fact, we can see that as the dimension gets very large, the error continues to decrease. For each entry in the following table we only did $10$ tests but otherwise using the same methodology as the previous table, but the standard errors remained small due to a significant lower variance:

\begin{table}[H]
\centering
\begin{tabular}{l|l|l|l}
N\textbackslash{}d & 100               & 200               & 400               \\ \hline
10,000             & 0.116 $\pm$ 0.002 & 0.092 $\pm$ 0.001 & 0.068 $\pm$ 0.000 \\
50,000             & 0.095 $\pm$ 0.002 & 0.079 $\pm$ 0.001 & 0.063 $\pm$ 0.000
\end{tabular}
\end{table}

\subsubsection{Model Selection}
So far, we have only considered the case where we actually know the dimension of the $\la=-1$ eigenspace in the symmetry matrix. This is unlikely to be the case in practice. Here we experimented with two of the approaches of model selection laid out in Section \ref{rankings}, in particular the (5 times repeated) 5-fold Cross Validation method and the full-dataset MMD method (we approximate it 5 times using batches of size $1,024$ and taking the mean of the results). For simplicity, we'll consider $d=10$ and $N=50,000$. We'll also use the Mean-Median mix statistic since it seems to perform the best.

Both of these approaches are quite slow, so we only use $50$ tests. The following tables gives the average results of these tests for choices of the bandwidth $h$.

\begin{table}[H]
\centering
\begin{tabular}{l|l|l}
                        & Mean Error $\pm$ SE & Correct Number of Swaps
\\\hline
Full Dataset ($h=1$) 		& 0.180 $\pm$ 0.020     & 38\% \\
Full Dataset ($h=3$)     	& 0.132 $\pm$ 0.021     & 62\% \\
Full Dataset ($h=6$)     	&  0.189 $\pm$ 0.024     & 46\% \\
Cross Validation ($h=1$) 	&  0.240 $\pm$ 0.019     & 25\% \\
Cross Validation ($h=3$) 	& 0.147 $\pm$ 0.022     & 54\%   \\
Cross Validation ($h=6$) 	&  0.153 $\pm$ 0.021     & 50\%                       
\end{tabular}
\end{table}

The increase in error is relatively small, and the majority of the errors in the methods occur when the ordering is already incorrect, so it seems that most of the difficulty comes from picking a good order. The difference in error between the full dataset and cross validation approaches seems relatively small, although selected an optimal bandwidth may require more work since the results seem more sensitive to this.

\subsubsection{Fine-Tuning}
The next important factor to consider is fine-tuning as described in Section \ref{finetuning-section}. In particular, for this experiment we will again assume we know the correct number of swapped eigenvectors (ie $5$). This process is significantly slower and more finicky than the previous methods, so we only considered $50$ synthetic datasets instead of $1000$, and used $100$ epochs. In fact, addition epochs seemed to often improve performance, but we kept it at $100$ to balance performance and accuracy. For Stochastic Gradient Descent we used a learning rate of $0.1$, and a momentum of $0.5$, with an additional penalty to push the change-of-basis matrix to remain orthogonal (although other experiments have indicated this may not be necessary). 

If there was numeric instability in the process and NaN values appeared, we would reduce the learning rate by a factor of $0.3$ and try again until it was able to finish a full $100$ epochs without issues. The results are plotted in Figure \ref{finetune-fig}.

\begin{figure}[H]
\centerline{\includegraphics[scale=0.7]{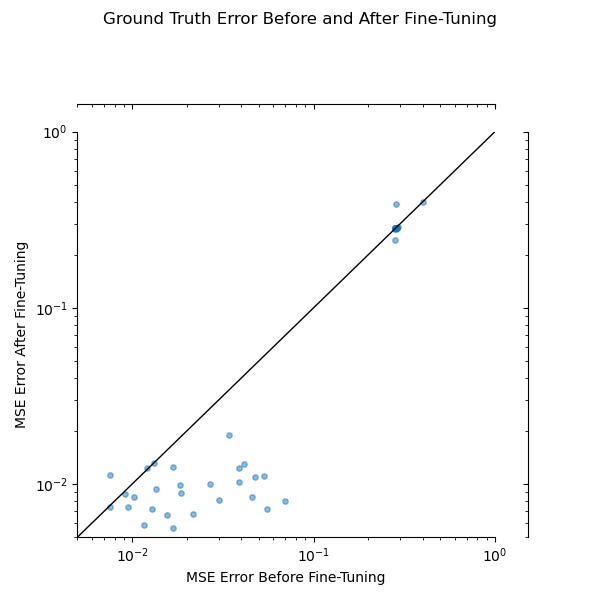}}
\caption{Error for $N=50,000$ and $d=10$ over $50$ tests after $100$ epochs of fine-tuning}
\label{finetune-fig}
\end{figure} 

This experiment seems to indicate that if the initial results are close enough to the global optima, then fine-tuning will have a large effect. In this case, fine-tuning almost always results in a mean error of roughly $0.01$ even if the error before fine-tuning was an order of magnitude larger. On the other hand, if the initial results are near one of the non-global optima, then little improvement should be expected from fine-tuning. Since image data is high dimensional, and we've seen that some incorrect selections are almost certain for high dimensions, this suggests that fine-tuning will have limited utility without additional techniques. In particular, the average error increased slightly due to a few outlier errors. 

\subsection{MNIST and EMNIST Datasets}
As a simple example of a more realistic dataset we will consider the MNIST and EMNIST datasets of handwritten numbers. The two primary reasons for this decision were that the dimension is smaller than most image datasets, and the dataset has some natural symmetries, for example the horizontal and vertical flips for labels like $0$ or $8$. However, after experimentation it turned out that these specific symmetries were not strong because handwritten letters tend to be tilted. 

In particular, we will consider a couple variants. First, we'll consider the semi-synethetic dataset where we include random horizontal flips. This allows us to see how well our techniques work with realistic data but with a known perfect symmetry. Second, we will consider subsets of the dataset with a fixed label, since these more restrictive sets are more likely to have strong symmetries. Finally, we will consider the full dataset, and attempt to apply our approach to produce data augmentation to improve supervised learning.

\subsubsection{Semi-Synthetic (Horizontal Flips)}
To analyze the effectiveness of our techniques we will consider a different metric from the synthetic datasets. Partly because the error in the real set appears to more often come from error in the covariance, and partly because the dataset is degenerate so it has no way of learning the full transformation. So we will directly compute the fraction of covariant eigenvectors are close to being eigenvectors of the horizontal flip transformation, and for those that are, we will determine which fraction are correctly swapped by our estimator transformation. 

For the former, for each eigenvector we will compute the angle between $v_{\la}$ and $T_{flip}v_{\la}$, and sort the results into three buckets. If the angle is less than 60 degree, or greater than 120 degrees, then it's close to an eigenvector of $T_{flip}$ with eigenvalue $+1$ or $-1$ respectively. Otherwise, we consider it to be in error.

For the latter, we take the half of (non-trivial) eigenvectors with the smallest statistic, then take those to be predicted unfixed vectors. We then compare this choice with where they were placed in the $+1$ and $-1$ buckets above, and calculate what fraction of vectors which are correctly swapped or fixed. 

To get a better sense of how dimension affects this, we resized the data to $4\times4$, $10 \times 10$, $16\times16$, $22 \times 22$ and $28 \times 28$ pixels. For half of each dataset, we did a horizontal flip, and left the other half alone. For the covariance eigenvectors, we got the following results:
\begin{table}[H]
\caption{Semi-Synthetic Covariance Eigenvector Accuracy}
\centering
\begin{tabular}{c|c|c|c|c|c}
      dataset \textbackslash{} dim & 4 & 10   & 16       & 22       & 28       \\
       \hline
MNIST  & 100\% & 80\% & 67\% & 57\% & 36\% \\
EMNIST & 100\% & 98\% & 89\% & 79\% & 62\%
\end{tabular}
\end{table}

There is a clear reduction in accuracy as the dimension increased. Moreover, there is a very clear improvement in accuracy from using the larger dataset. Nevertheless, even for the EMNIST dataset it appears that the accuracy for the full sized images isn't sufficient.

We get a similar result for selecting the correct eigenvectors. In low dimensions, we get very high accuracy, which quickly deteriorates as the dimension increases. 

\begin{table}[H]
\caption{MNIST Semi-Synthetic Eigenvector Selection Accuracy}
\centering
\begin{tabular}{c|c|c|c|c|c|c}
  dim \textbackslash{} stat & mean    & median  & mm mix  & sign    & corr adj & label-based \\
   \hline
4  & 88\% & 88\%   & 88\%    & 88\% & 88\% & 100\%         \\
10 & 71\% & 61\%   & 69\%    & 84\% & 74\% & 90\%        \\
16 & 68\% & 57\%   & 64\%    & 76\% & 68\% & 81\%        \\
22 & 64\% & 58\%   & 62\%    & 66\% & 64\% & 68\%         \\
28 & 63\% & 57\%   & 58\%    & 62\% & 61\% & 65\%
\end{tabular}
\end{table}

First, it should be noted that the label-based clustering approach makes use of the labels, so while it is more accurate, it is also less useful for semi-supervised learning. Putting that column aside, we get a surprising result. In the semi-synthetic data the sign statistic generally outperforms the other statistics, including the correlation adjusted mean. This is nearly the opposite of the synthetic data where the sign statistic did very poorly, while the correlation adjusted mean dominated. The difference in the performance of the sign statistic seems likely to be due to the disconnected nature of the synthetic data. In that case, the sign statistic may detect that the two sides are balanced, but not notice that one cluster is farther from $0$ than the other. 

Moreover, the ground truth error keeps increasing, unlike the corresponding synthetic datasets, suggesting that the synthetic datasets are missing some important properties of more realistic large dimensional datasets.

For the EMNIST dataset, we get similar results. However, unlike with the covariance eigenvectors, the improvement due to the larger dataset is significantly smaller in the larger dimensions.

\begin{table}[H]
\caption{EMNIST Semi-Synthetic Eigenvector Selection Accuracy}
\centering
\begin{tabular}{c|c|c|c|c|c|c}
   dim \textbackslash{} stat & mean  & median & mm mix & sign  & corr adj & label-based \\
     \hline
4  & 100\% & 100\%  & 100\%   & 100\% & 100\% & 100\%        \\
10 & 86\%  & 77\%   & 83\%    & 89\%  & 86\%  & 92\%        \\
16 & 78\%  & 65\%   & 74\%    & 78\%  & 76\%  & 83\%        \\
22 & 68\%  & 59\%   & 66\%    & 69\%  & 68\%  & 69\%        \\
28 & 67\%  & 62\%   & 65\%    & 69\%  & 68\%  & 73\%   
\end{tabular}
\end{table}

Model selection seems to be tricky in this case. Even in the $10 \times 10$ case, the dimensions that are nearly trivial can produce serious problems, as seen in Figure \ref{SS_MDD_Error}. There is a noticeable dip at $50$ as there should be (with the actual local minimum at $51$), but it is relatively shallow, and the first couple dimensions have lower error because their variance is so small.

\begin{figure}[H]
\centerline{\includegraphics[scale=0.75]{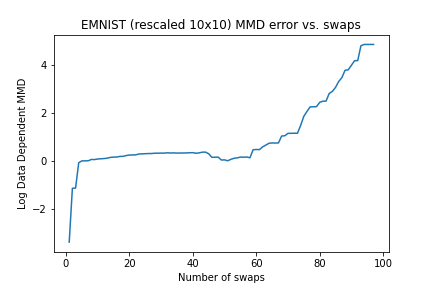}}
\caption{EMNIST Data-based MMD Errors using Sign-Based Ranking. $\sigma^2=5$, $\la=0.03$, and using 25\% of the data as a validation set to do model selection}
\label{SS_MDD_Error}
\end{figure}

The dip seems to be more clear in lower dimensions, and can disappear in higher dimensions, suggesting we may need to use alternative methods in order to perform model selection when the dimension is large.

\subsubsection{Fixed Label}

Next we consider the unmodified (E)MNIST dataset, but restricted to a specific label where symmetries should be easier to find. The results of the semi-synthetic test suggest that the dimension of the full $28 \times 28 = 784$ dimensional space is probably too large for accurate selections without an even larger dataset, or incorporating other priors like spacial continuity. Therefore, we consider the $10 \times 10$ case where it is possible to see what's happening visually while still keeping a relatively small dimension.

Secondly, as with most unsupervised learning methods, the model is rarely perfect, so we may need to add more flexibility to model selection methods. In particular, it is sometimes the case that very low covariance dimensions contribute very little to the scoring, and so swapping them is roughly the same as doing nothing, causing the model selection to erroneously just pick one of those vectors to negate. 

If we use the spherical MMD, and just look at the EMNIST dataset of points labeled $0$ we get the following graph of MMD versus the number of eigenvectors negated:

\begin{figure}[H]
\centerline{\includegraphics[scale=0.75]{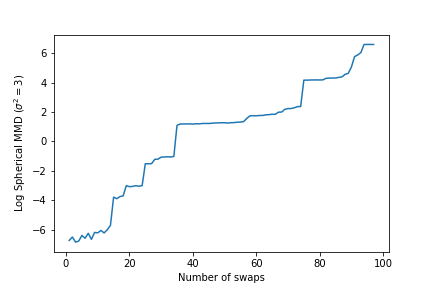}}
\caption{EMNIST MMD Errors using Sign-Based Ranking}
\label{0_spherical}
\end{figure}

There are no large dips, and instead the error is dominated by places where eigenvectors with large eigenvalues are added.  We can address this somewhat by rescaling the MMD so that dimensions with larger variances are adjusted down as explained in subsection \ref{rankings}.
\begin{figure}[H]
\centerline{\includegraphics[scale=0.75]{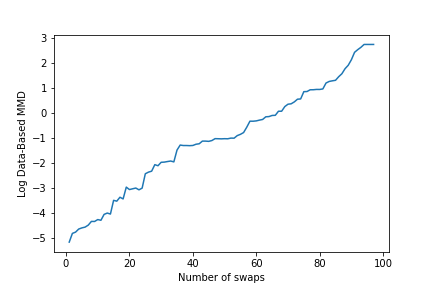}}
\caption{EMNIST MMD Errors using Sign-Based Ranking. $\sigma^2=5$ and $\la=0.03$}
\label{0_non_spherical}
\end{figure}

This greatly reduces the size of the large jumps, but we still don't get any significant dips. Part of this is likely due to the innate bias in the ranking method. Each time we add a dimension to swap, the transformation necessarily respects the mean and standard deviation less, so reductions in MMD need to come from reductions in larger moments. However, especially as the bandwidth increases, the MMD focuses more on the lower order moments, making it very difficult for a drop to appear in the graph. Unfortunately, setting the bandwidth to be small is often infeasible because the scoring becomes too unstable.

We can graph the effect of the learned transformation on a sample of the dataset, with the top row being the original images, and each subsequent row involving more and more change.

\begin{figure}[H]
\centerline{\includegraphics[scale=0.60]{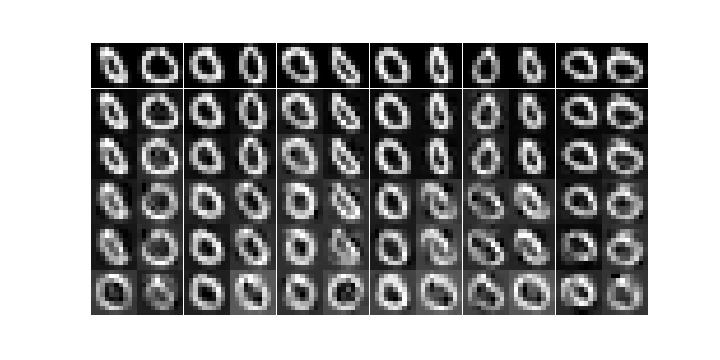}}
\caption{Potential transformations of $0$. From top to bottom $0$, $13$, $23$, $45$, $65$, and $89$ swaps}
\label{0s}
\end{figure}

We see that the first $3$ rows there is very little difference, but that for the $4$th and $5$th we start to see some changes. In particular, some of columns (especially the $4$th, $9$th and $10$th) seem to switch between roughly vertical, and slanted top-left to bottom-right. We also see some errors begin to appear, in particular in the $8$th column. We see a similar result for images labeled $1$

\begin{figure}[H]
\centerline{\includegraphics[scale=0.60]{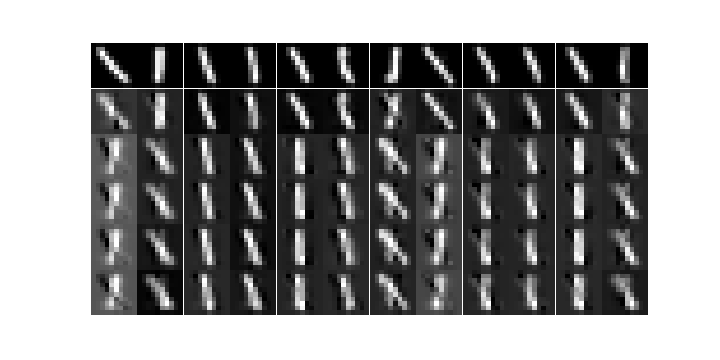}}
\caption{Potential transformations of $1$. From top to bottom $0$, $20$, $40$, $50$, $60$, and $80$ swaps}
\label{1s}
\end{figure}

On the other hand, some labeled classes like $4$ seem to be too complicated for this method, and produce incoherent transformations. It seems to be again try to switch from straight to slanted, but in the process picks up a lot of blurring. 

\begin{figure}[H]
\centerline{\includegraphics[scale=0.65]{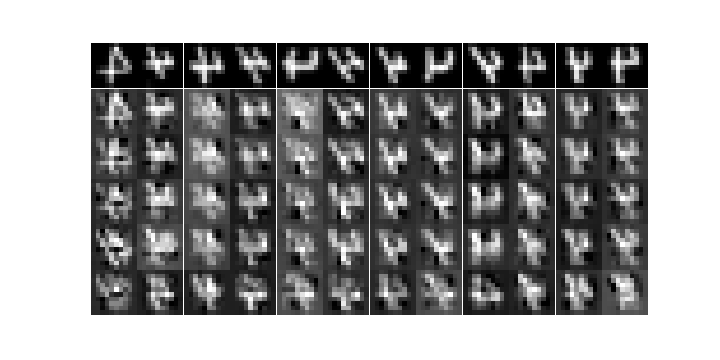}}
\caption{Potential transformations of $4$. From top to bottom $0$, $20$, $40$, $50$, $60$, and $80$ swaps}
\label{4s}
\end{figure}

\subsubsection{Full Dataset}

Ideally we hope to be able to find a useful symmetry that works for all labels simultaneously, because this could allow us to use it to do semi-supervised data augmentation. Unfortunately, in this case even only a couple of swapped dimensions quickly leads to large errors. This is somewhat unsurprising given that this process didn't work with some of the individual labels.

\begin{figure}[H]
\centerline{\includegraphics[scale=0.85]{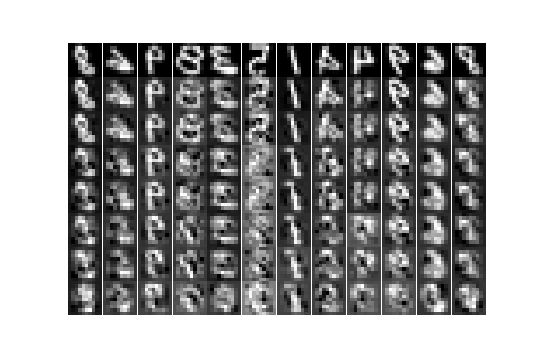}}
\caption{Potential transformations of full digits EMNIST dataset. From top to bottom $0$, $5$, $10$, $15$, $20$, $40$, $60$ and $80$ swaps}
\label{alls}
\end{figure}

Clearly, additional work needs to be done to make this functional for more complicated datasets. Since the transformation seems to make some of the images change the label (for example in column $1$), it seems to be necessary to force the symmetry to respect each single-label dataset individually. Moreover, since the algorithm has trouble with some labels, it's probably also necessary to allow more flexible transformations to be learned. 

\section{Further Directions}

\subsection{More-Global Scoring}
One of the key limitations is that scoring of dimensions is primarily local. We obtain a ranking of dimensions, then just find a cut-off point. However, if there are multiple linear symmetries, or if we accidentally incorrectly rank some of the dimensions, there may be no good symmetry that can be obtained by selecting the top ranked dimensions. Alternative approaches such as greedily selected dimensions has the potential to be more robust to changes in the model assumptions.

\subsection{Non-Linear Symmetries}
The current paper focused on the case of linear symmetries because they are relatively easy to work with, and because many symmetries in images should be locally linear. However, in order to achieve the potential of this goal, it will be necessary to learn symmetries which are not linear or affine, since few symmetries are likely to be globally linear. 

There are a number of ways we could try to leverage this work to the non-linear setting. If the desired symmetry is close to linear, we could initialize a more flexible model at the linear transformation and then fine tune it. Alternatively, we could learn a symmetry within the encoding of a non-linear generative model.

\subsection{Higher Dimensional Symmetries}
The methods in the current paper had difficulty when the dimension $d$ grew much larger than $100$. Ideally, we would like to be able work with larger images of order $d \approx 300 \times 300 \approx \num{100000}$. Beyond the accuracy issues, this will start to run into computational issues since we'd need to compute the covariance which would be $\num{100000} \times \num{100000} = \num{10000000000}$ dimensional, which starts to become impractical. 

The most direct solution is to use dimension reduction techniques. In particular, we never used the extra structure arising from translation and dilation symmetry. Enforcing these strictly would restrict the allowable linear symmetries to reflections and rotations of the image, but weaker constraints like penalties on discontinuity or failure to respect rescaling may greatly improve performance in higher dimensions. 

\subsection{Higher Order Symmetries}
In this paper, we restricted our attention to the case where our models had only order $2$ symmetries, so applying the transformation twice gave back the identity transformation. This restriction was justified by the fact that we can identify datasets which might have higher order symmetries by looking for eigenvalues that are approximately equal. An important extension would be to develop techniques to better handle cases where some eigenvalues are equal. If there are too many of these, we will start to fall into the intractable setting, but dealing with a smaller number of these equalities should be feasible. One would need to identify which eigenvalues are likely equal, perhaps determine a good basis for the corresponding eigenspace, then use this to determine if the eigenspace has higher order or even continuous symmetries.

\printbibliography[heading=bibintoc, title={References}]

\end{document}